\documentclass[twoside,leqno,twocolumn]{article}
\usepackage{ltexpprt}

\usepackage{afterpage} 

\usepackage{tcolorbox}

\usepackage[square,comma,numbers,sort]{natbib}
\usepackage{mathtools} 
\usepackage{amssymb} 
\usepackage{bm} 

\usepackage{etoolbox} 
\AtEndEnvironment{proof}{\null\hfill\ensuremath{\Box}}%

\newproof{BRE}{The Base Rate Example}
\newproof{PE}{The Peculiar Example}

\usepackage{graphicx}
\usepackage{subfig}
\usepackage{hyperref}
\usepackage{booktabs}

\let\Pr\relax
\DeclareMathOperator{\Pr}{\mathbb{P}}
\DeclareMathOperator{\I}{\mathbb{I}} 
\DeclareMathOperator{\E}{\mathbb{E}}
\DeclareMathOperator{\Dir}{Dir}
\DeclareMathOperator{\diag}{diag}
\DeclareMathOperator{\Var}{Var}

\newcommand{\R}{\mathbb{R}} 

\title{\Large A Bayesian Approach for Accurate Classification-Based Aggregates}

\author{
    Q.A. Meertens \thanks{Corresponding author; q.a.meertens@uva.nl}
    \thanks{CeNDEF, University of Amsterdam}
    \thanks{LCDS, Leiden University}
    \thanks{Statistics Netherlands, The Hague}
    \and
    C.G.H. Diks
    \footnotemark[2]
    \thanks{IAS, University of Amsterdam}
    \hspace{1pt}
    \thanks{Tinbergen Institute, Amsterdam}
    \and
    H.J. van den Herik
    \footnotemark[3]
    \and
    F.W. Takes
    \thanks{LIACS, Leiden University}
    \hspace{1pt}
    \thanks{CORPNET, University of Amsterdam}
}
\date{}

\begin{document}

\maketitle







\pagenumbering{arabic}
\setcounter{page}{1}

\afterpage{\cfoot{\thepage}}

\begin{abstract}
    \small\baselineskip=9pt
    In this paper, we study the accuracy of values aggregated over classes predicted by a classification algorithm. The problem is that the resulting aggregates (e.g., sums of a variable) are known to be biased. The bias can be large even for highly accurate classification algorithms, in particular when dealing with class-imbalanced data. To correct this bias, the algorithm's classification error rates have to be estimated. In this estimation, two issues arise when applying existing bias correction methods. First, inaccuracies in estimating classification error rates have to be taken into account. Second, impermissible estimates, such as a negative estimate for a positive value, have to be dismissed. We show that both issues are relevant in applications where the true labels are known only for a small set of data points. We propose a novel bias correction method using Bayesian inference. The novelty of our method is that it imposes constraints on the model parameters. We show that our method solves the problem of biased classification-based aggregates as well as the two issues above, in the general setting of multi-class classification. In the empirical evaluation, using a binary classifier on a real-world dataset of company tax returns, we show that our method outperforms existing methods in terms of mean squared error.
\end{abstract}

\vspace{\baselineskip} \noindent \textbf{Keywords:} Bayesian inference, classification errors, aggregates, misclassification bias, machine learning


\section{Introduction}\label{sec_intro}

Aggregation after automated classification naturally occurs in a wide range of data mining applications. Classification-based aggregation even is the most common data operation in some research fields. Therefore, comprehending the effect of classification errors on the accuracy of the resulting aggregates is essential. The aim of this paper is to improve the accuracy of such aggregates by reducing their statistical bias. Before discussing technical details, we briefly describe two applications, demonstrating the relevance of the problem.

The first application is sentiment analysis in social media \cite{daas_big_2015, ravi_survey_2015}. For the sake of simplicity, assume that messages are either positive or negative and that the overall sentiment is defined as the difference between the number of positive and negative messages. The sentiment of one message is predicted using natural language processing. The estimator for the sentiment on social media obtained in this manner is statistically biased, unless precision equals recall, as we show below.

The second application is land cover mapping based on satellite imagery \cite{costa_land_2018, ma_review_2017}. Here, the aim is to estimate the area of different types of land cover (e.g., cropland, wetland) of a large, delimited surface (e.g, a country, a continent). In \cite{costa_land_2018}, the per-pixel land cover class (one of 15) is predicted by an SVM image classifier. Again, the estimated surface per land cover class is biased.

The two applications above are examples of the following general setting. Consider a set of $N$ data points. Each data point is equipped with a categorical variable $s$ (for \textit{stratum}) and a numerical variable $y$. We are interested in the aggregates obtained by summing $y$ after grouping by $s$. We denote the resulting aggregates by the $K$-vector $\bm{u}$, where $K>1$ is the number of categories that $s$ may attain. Now, if $s$ is not observed, but the result of a classification algorithm is used instead, the latter may contain classification errors. We therefore distinguish between the true class $s$ and the predicted class $\widehat{s}$. Similarly, we write $\widehat{\bm{u}}$ for the $K$-vector of \textit{classification-based aggregates} obtained by summing $y$ after grouping by $\widehat{s}$.

The problem studied in this paper is that the vector of classification-based aggregates $\widehat{\bm{u}}$ will be a \textit{statistically biased} estimator for the true aggregate vector $\bm{u}$. The bias can be nonzero even if $N$ is large and the accuracy of the classification algorithm is high. In fact, the example below shows that the bias does not depend on $N$ at all.

\begin{BRE}
    We consider a set of $N~=~100,\!000$ companies, in which we would like to identify webshops based on the text found on the company's website. Assume a trained classification algorithm with false negative rate $p = \text{FN}/(\text{TP}+\text{FN}) = 0.01$ and false positive rate $q = \text{FP}/(\text{TN}+\text{FP}) = 0.005$. Assume that the set contains $u_1 = 10,\!000$ webshops, which would in practice be unknown. The fraction $u_1/N = 0.1$ is referred to as the \textit{base rate}. The expected number of companies classified as webshops is
    \begin{equation}\label{eq:intro_example}
        0.99 \cdot 10,\!000 + 0.005 \cdot 90,\!000 = 10,\!350,
    \end{equation}
    showing that the estimator has a relative bias of +3.5\%.
\end{BRE}

Essentially, the base rate example shows that if $s$ is binary and  $y \equiv 1$,
\begin{equation}\label{eq:main_problem}
    \E[\widehat{\bm{u}}] = P^T\bm{u} \quad \text{with} \quad
    P = 
    \begin{pmatrix}
        1-p & p \\
        q & 1-q    
    \end{pmatrix}.
\end{equation}
In fact, it can be shown that equation~\eqref{eq:main_problem} holds for a multi-class variable $s$ and any numerical variable $y$, with $P$ being the $K\times K$ \textit{contingency matrix} \cite{van_delden_accuracy_2016}. We now make two crucial observations. First, the \textit{relative bias} $\E[(\widehat{\bm{u}} - u)/N]$ does \textit{not} depend on $N$. It implies that the bias does not vanish for large datasets. Second, the bias is only equal to 0 if $p=q=0$, or if (for constant $y$) the base rate is precisely equal to $q/(p+q)$. In the case of binary classification and constant $y$, the latter is equivalent to precision being equal to recall. In the case of multi-class classification and constant $y$, it is equivalent to the base rate vector (and hence $\bm{u}$) being an eigenvector (of the transposed contingency matrix $P^T$) corresponding to the eigenvalue $1$.

Now, if the inverse $Q = (P^T)^{-1}$ of $P^T$ is well-defined (in the binary case: if $p+q \neq 1$), then $Q\widehat{\bm{u}}$ is an unbiased estimator for $\bm{u}$. The problem is that the classification error rates are not known exactly and merely estimated. If the estimates of the classification error rates are based on a (very) small test set, the proposed unbiased estimator $Q\widehat{\bm{u}}$ might attain impermissible values. 
To illustrate this problem, we include a second example.

\begin{PE}
    Consider a set of $N=100$ companies in which the \textit{predicted} number of webshops is equal to $10$. To estimate the classification errors, a (very) small test set of size $n=10$ is used. Assume that it resulted in $\text{TP} = 4$, $\text{FP} = 1$, $\text{FN} = 2$ and $\text{TN} = 2$, hence $p = 0.2$ and $q = 0.4$. Correcting the bias as suggested above yields
    \begin{equation}
        \left(P^T\right)^{-1}\widehat{\bm{u}} =
        \begin{pmatrix}
            1.5 & -1 \\
            -0.5 & 2
        \end{pmatrix}
        \begin{pmatrix}
            10 \\ 90
        \end{pmatrix}
        =
        \begin{pmatrix}
            -75 \\ 175
        \end{pmatrix}.
    \end{equation}
    Thus, the unbiased estimate of the number of webshops in the dataset is $-75$.
\end{PE}

The issue in the peculiar example is caused by the fact that (1) $p$ and $q$ are not known but merely estimated and (2) the base rate is relatively low. Observe that the outcome does not dependent on the size $N$ of the \textit{full} dataset. The problem arises because the \textit{test} set is small (because it implies inaccurate estimates for $p$ and $q$), as we will show in Section~\ref{sec_methods_bayes_restr}. Having only a small \textit{test} set available, even if $N$ is large, is quite common. The reason is that labeled data is unavailable in many applications (e.g., the sentiment analysis and land cover mapping examples), while manually creating labeled data requires expert knowledge, making it expensive.

We propose a novel bias correction method that reduces the statistical bias of classification-based aggregates. In contrast to existing methods, it is suitable for applications where the test set is small. If the test set \textit{is} large enough, our bias correction method will give as accurate results as existing methods.

This paper is structured as follows. In Section~\ref{sec_relwork}, the formal problem statement is introduced and related work is discussed. In Section~\ref{sec_methods}, we formulate our bias correction method in the general setting of multi-class classification problems. In Section~\ref{sec_exper}, we illustrate the effectiveness of the proposed methods using experiments on real-world data. Finally, Section~\ref{sec_concl} concludes.


\section{Problem Statement and Related Work}\label{sec_relwork}

In this section, we introduce the notation and formal problem statement, and discuss related work.

\subsection{Problem Statement.}\label{sec_relwork_back}

We partly adopt the notation from \cite{van_delden_accuracy_2016} to formulate our problem statement. Consider a set $I$ of objects $i = 1, 2, \ldots, N$. Each object $i \in I$ belongs to a class $s_i$. The finite, ordered set of classes is denoted by $H$ and is of size $K \coloneqq |H| \geq 2$. In addition, each object is attributed with a continuous variable $y_i$ of interest. We introduce the \emph{class matrix} $A = (a_{ih})$ of dimension $N\times K$ given by $a_{ih} = \I(s_i = h)$, where $\I(\cdot)$ denotes the indicator function. The \textit{counts vector} $\bm{v} \coloneqq A^T \bm{1}$, where $\bm{1}$ is an $N$-vector of ones, counts the number of occurrences of the classes $h \in H$ in the population. That is, $v_h$ equals the number of $i \in I$ for which $s_i = h$. The \textit{base rate vector} $\bm{\beta}$ of length $K$ is given by $\beta_h = v_h/N$ for $h \in H$.

Assume that the classes $s_i$ are not known, but instead predicted to be $\widehat s_i$. The predicted class matrix based on the predictions $\widehat{s_i}$ is denoted by $\widehat A$. Similarly, $\widehat{\bm{v}}$ denotes the predicted counts vector. Assume that $y_i$ is known for all $i \in I$. The goal is to estimate the sum of $y_i$ over each of the classes $h \in H$. In other words, the main problem statement is how to find, using the predicted $\widehat s_i$ and the known $y_i$, an accurate estimator of the \textit{aggregate vector}
\begin{equation}\label{eq_background_target}
    \bm{u} = A^T y.
\end{equation}
In the sentiment analysis application from before, the set $I$ is the set of messages and $s_i$ is binary (a positive or negative message). Each $y_i$ equals 1 and hence $\bm{u}$ is a 2-vector with $u_1$ the number of positive messages and $u_2$ the number of negative messages. In the land cover mapping application obtained from \cite{costa_land_2018}, each $i \in I$ is a pixel and $s_i$ may attain $K=15$ different values. The value $y_i$ is the real land area corresponding to pixel $i$. Hence, the 15-vector $\bm{u}$ contains the total land area of each of the 15 land cover classes.

Now, a first estimator of $\bm{u}$ might be the $K$-vector
\begin{equation}\label{eq_background_predicted}
    \widehat{\bm{u}} = \widehat{A}^Ty.
\end{equation}
However, we know that $\widehat{\bm{u}}$ is a biased estimator of $\bm{u}$, recall equation~\eqref{eq:main_problem}. To estimate (and then correct) this statistical bias, we assume a \emph{classification error model}, following the methodology in \cite{van_delden_accuracy_2016}; the value of $\widehat s_i$, given the value of $s_i$, is assumed to be a stochastic variable following a \textit{categorical distribution}. The stochastic variable depends on $i$, but draws for different $i$ are assumed to be independent. The unknown event probabilities are denoted by $p_{ghi} = \Pr(\widehat s_i = h \mid s_i = g)$ and stored in \textit{contingency matrices} $P_i = (p_{ghi})_{gh}$ of dimension $K\times K$. The suggested estimator $\widehat{\bm{u}}$ is shown to have expectation $\E[\widehat{\bm{u}}] = \sum_{i\in I} P_i^T \bm{\widehat{a}_i} y_i$, where $\bm{\widehat{a}_i}$ is the $i$-th column of $\widehat{A}^T$. If $P_i^T$ is invertible, we denote its inverse by $Q_i$. An unbiased estimator of $\bm{u}$ is then given by $\sum_{i\in I} Q_i \bm{\widehat{a}_i} y_i$. We refer to this bias correction method as \textit{the baseline method}.

What remains is to estimate the classification error rates $p_{ghi}$ using a test set. In Section~\ref{sec_methods}, we propose a novel estimation method that properly deals with small test sets (recall The Peculiar Example).

\subsection{Related Work.}\label{sec_relwork_lit}
Below, we discuss related work on biased aggregates, bias in machine learning and Bayesian inference.

\vspace{\baselineskip} \noindent \textbf{Biased Aggregates.} To the best of our knowledge, \cite{van_delden_accuracy_2016} is the first work \textit{generally} describing the classification error model and studying classification-based aggregates. We mention front runners from two fields using similar bias correction methods.

The first field is \textit{epidemiology}, which studies the distribution of health and disease. There, it is well-known how a low base rate can lead to high bias even if sensitivity ($1-p$) and specificity ($1-q$) are high \cite{lash_applying_2009} (cf. The Base Rate Example). As mostly binary classifiers are considered (sick or not), bias correction is straightforward \cite[pp.~87-89]{lash_applying_2009}. In addition, a standard Bayesian method of bias correction, predominantly using a uniform or Jeffreys prior without parameter constraints, is applied in epidemiology \cite{gustafson_measurement_2004, goldstein_bayesian_2016}.
We will generalize these Bayesian methods to the entire family of conjugate prior distributions and to the setting of multi-class classification. Moreover, we will improve empirical performance by imposing well-chosen parameter constraints.

The second field is \textit{land cover mapping}, which analyzes land use based on large volumes of remote sensing data, for example using SVM \cite{low_analysis_2015}. As can be expected, multi-class classification is not uncommon in this subject area (see \cite{costa_land_2018} for a case study with $K=15$ classes). Since accurately estimating the total area of types of vegetation is highly relevant in monitoring ecological systems \cite{veran_modeling_2012}, there have been many efforts in the field to make better use of accuracy data \cite{olofsson_making_2013}. To the best of our knowledge, our bias correction method is a novel contribution to these efforts.

\vspace{\baselineskip} \noindent \textbf{Bias in Machine Learning.} In machine learning, the accuracy of classification-based aggregates is relatively understudied. The literature on machine learning is mainly concerned with minimizing loss for \textit{individual} future predictions and therefore focuses on a different kind of bias. Much work deals with model selection bias and overfitting \cite{cawley_over-fitting_2010} and sample selection bias \cite{zadrozny_learning_2004}. However, reducing these types of bias (by, e.g., using $k$-fold cross-validation) does not necessarily reduce the \textit{statistical} bias of classification-based aggregates. This is a pitfall especially if the base rate is low, i.e., when dealing with \textit{class-imbalanced data}. Of course, correctly dealing with class-imbalanced data is well-studied \cite{haixiang_learning_2017}. However, as long as the classifier is not error-free, it will still result in biased aggregate predictions.

An alternative is an \textit{aggregate loss function} that measures the bias on the aggregate level instead of on the individual level \cite{sodomka_predictive_2013}. 
Other alternatives include averaging multiple (biased) estimators into a single, more accurate estimator \cite{taniguchi_averaging_1997}. Such alternatives will reduce the bias of classification-based aggregates, but our proposed method will completely remove it for sufficiently large test sets.

\vspace{\baselineskip} \noindent \textbf{Bayesian Inference.} An extensive review on Bayesian inference for categorical data analysis can be found in \cite{agresti_bayesian_2005}. It specifically comments on the use of prior distributions and ``the lack of consensus about what noninformative means'' \cite{agresti_bayesian_2005}. We will avert this discussion by analytically deriving the posterior distribution for the entire family of conjugate priors. We empirically evaluate two common prior choices; the uniform (flat) prior and the Jeffreys prior. In real-world applications, our bias correction method can be implemented for non-conjugate prior as well, by utilizing Markov chain Monte Carlo (MCMC) methods.


\section{Methods}\label{sec_methods}

In this section, we introduce our bias correction method in the general setting of multi-class classification problems. The section contains three parts. First, the likelihood and posterior (for conjugate priors) as well as the Jeffreys prior are shown. Second, we formulate novel constraints on the classification error rates, being our main scientific contribution. Third, we show how to obtain the proposed Bayesian estimator of the aggregate vector using these parameter constraints.

\subsection{Bayesian Parameter Estimation.}\label{sec_methods_bayes} We begin by translating existing theorems from Bayesian statistics to our setting of multi-class classification. It mostly contains elementary probability manipulations. We make one simplifying assumption compared to \cite{van_delden_accuracy_2016}: the probabilities $p_{ghi}$ do not depend on $i$. We write $p_{gh}$ instead and only a single contingency matrix $P$ remains. We refer to this assumption as that of the \textit{homogeneity of the contingency matrices}. The reason for this simplifying assumption is that the notation, derivations and formulas are more pleasant to read. In practice, it might be more reasonable to assume that $p_{ghi}$ depends on $i$, but only through $y_i$ and possibly other features. Our proposed methodology can be applied separately to each group of objects having similar features, which can then be aggregated to obtain a single, final estimate.

Two parts now follow: (1) formulating the likelihood function and posterior distribution (for conjugate priors), and (2) deriving the Jeffreys prior.

\vspace{\baselineskip} \noindent \textbf{Likelihood and Posterior.} The parameters of the classification error model are the $K^2$ classification error rates (or \textit{event probabilities}) $\{p_{gh}: g,h \in H\}$ and the $K$ base rate parameters $\{\beta_g : g \in H\}$. To estimate these parameters, consider a test set $J = \{i_1, \ldots, i_n\} \subset I$ of $n$ randomly selected $i_j \in I$ for which we observe $s_{i_j}$. The corresponding dataset $\mathcal{D}$ of $n$ independent observations $\bm{x}_1, \ldots, \bm{x}_n$ is given by $\bm{x}_j = (s_{i_j}, \widehat s_{i_j})$, for $i_j \in J$. We simply write $\bm{x}_j = (s_j, \widehat s_j)$ for $j=1, \ldots, n$. The following theorem shows the likelihood function and the posterior distribution for a suitable family of prior distributions, namely Dirichlet distributions. Recall that a Dirichlet distribution of order $k \geq 2$ has the standard $(k-1)$-simplex $\Delta_{k-1} \subset \mathbb{R}^k$ as support, where
\begin{equation}
    \Delta_{k-1} = \left\{\bm{x} \in \mathbb{R}^k : x_i \geq 0, {\textstyle \sum}_i x_i = 1\right\}.
\end{equation}
The density of a Dirichlet distribution of order $k \geq 2$ with concentration parameters $\bm{\alpha} = (\alpha_1, \ldots, \alpha_k)$ equals
\begin{equation}\label{eq_def_Dir}
    f( \bm{x} \mid \bm{\alpha}) = \frac{\Gamma(\alpha_1 + \cdots + \alpha_k)}{\Gamma(\alpha_1) \cdots \Gamma(\alpha_k)} \prod_{m=1}^k x_m^{\alpha_m-1},
\end{equation}
where $\Gamma(\cdot)$ is the gamma function. A Dirichlet distribution will be referred to as $\Dir(k, \bm{\alpha})$.

\begin{theorem}\label{thm_like_post}
    The likelihood function of observing the dataset $\mathcal{D} = \{\bm{x}_1, \ldots, \bm{x}_n\}$, given the model parameters $\bm{p}$ and $\bm{\beta}$, is given by
    \begin{equation}\label{eq_likelihood}
        p(\mathcal{D} \mid \bm{p}, \bm{\beta}) = \prod_{g, h \in H} \left(p_{gh}\beta_g\right)^{n_{gh}},
    \end{equation}
    where $n_{gh} = |\{j \in J: x_j = (g, h)\}|$. The family consisting of products of $K+1$ independent Dirichlet distributions of order $K$ forms the collection of conjugate priors for the above likelihood function. Next, choose the prior $\prod_{k=1}^{K+1} D_k$ on $(\bm{p}, \bm{\beta})$, where $D_g \sim \Dir(K, \bm{\alpha}_g)$ for $g \in H$ and $D_{K+1}\sim\Dir(K, \bm{\gamma})$. Write $\bm{\alpha}_g = (\alpha_{gh})_{h\in H}$ for $g\in H$ and $\bm{\gamma} = (\gamma_g)_g$. The posterior density is then given by
    \begin{equation}\label{eq_post}
        p(\bm{p}, \bm{\beta} \mid \mathcal{D}, \bm{\alpha}, \bm{\gamma})
        \hspace{-3pt} \propto \hspace{-3pt}
        \left(\prod_{g, h \in H} p_{gh}^{\alpha_{gh}+n_{gh}-1}\right) \hspace{-5pt}
        \left(\prod_{g \in H} \beta_g^{\gamma_g + n_g -1}\right) \hspace{-3pt},
    \end{equation}
    where $n_g = \sum_{h\in H} n_{gh}$. Thus, the posterior distribution is the product of $K+1$ independent Dirichlet distributions, the first $K$ being $\Dir(K, (\alpha_{gh} + n_{gh})_{h \in H})$, for $g\in H$, and the last one being $\Dir(K, (\gamma_g + n_g)_{g \in H})$.
\end{theorem}

\begin{proof}
    The probability that $s_i = g$ for a randomly selected $i \in I$ equals $\beta_g$. The probability of observing $\bm{x}_i = (g, h)$, given the classification error model, is equal to
    \begin{equation}
        \Pr(s_i = g, \widehat s_i = h) = \Pr(\widehat s_i = h \mid s_i = g)\Pr(s_i = g) = p_{gh}\beta_g.
    \end{equation}
    The pairs $\bm{x}_i$ again follow a categorical distribution, now mapping into the product space $\{1, \ldots, K\} \times \{1, \ldots, K\}$ having event probabilities $p_{gh}\beta_g$. Equation~\eqref{eq_likelihood} now follows, see \cite[pp.~74-75]{bishop_pattern_2006}.
    
    Next, we observe that the likelihood function \eqref{eq_likelihood} is equivalent to
    \begin{equation}\label{eq_likelihood_alt}
        p(\mathcal{D} \mid \bm{p}, \bm{\beta}) =
        \left(\prod_{g, h \in H} p_{gh}^{n_{gh}}\right)
        \left(\prod_{g \in H} \beta_g^{n_g}\right),
    \end{equation}
    where $n_g = \sum_h n_{gh} = |\{j \in J: s_j = g\}|$. The likelihood function can be viewed as a product of $K+1$ independent categorical distributions, as $\sum_h p_{gh}~=~1$ for every $g \in H$. It is well-known that the family of Dirichlet distributions forms the collection of conjugate priors to the categorical (and multinomial) distribution. For a derivation, see \cite[pp.~76-78]{bishop_pattern_2006}. This proves the claim regarding the family of conjugate priors for the likelihood function \eqref{eq_likelihood}.
    
    Finally, to derive the posterior distribution, we apply the result from \cite[pp.~76-78]{bishop_pattern_2006} to each of the parameter vectors $\bm{p}_g$ (with $g \in H$) and $\bm{\beta}$ separately. The posterior distribution for a prior distribution $\prod_{k=1}^{K+1} D_k$ on $(\bm{p}, \bm{\beta})$, where $D_g \sim \Dir(K, \bm{\alpha}_g)$ for $g \in H$ and $D_{K+1}\sim\Dir(K, \bm{\gamma})$, is then seen to be equal to 
    \begin{equation*}
        p(\bm{p}, \bm{\beta} \mid \mathcal{D}, \bm{\alpha}, \bm{\gamma})
        \hspace{-3pt} \propto \hspace{-3pt}
        \left(\prod_{g, h \in H} p_{gh}^{\alpha_{gh}+n_{gh}-1}\right) \hspace{-5pt}
        \left(\prod_{g \in H} \beta_g^{\gamma_g + n_g -1}\right) \hspace{-3pt}.
    \end{equation*}
    This concludes the proof.
\end{proof}

\vspace{\baselineskip} \noindent \textbf{The Jeffreys Prior.}
If the model is correctly specified, the effect of prior choices will diminish as $n$ attains larger values. For small test set sizes $n$, the choice of a prior distribution does affect the posterior distribution. With the breakthrough of MCMC methods, there is no need to impose any restrictions (other than those resulting from numerical limitations) on the family of prior distributions. For now, we have considered conjugate prior distributions in order to analytically formulate the posterior distribution. Specifically, we will compare two common prior choices: (1) the uniform prior and (2) the Jeffreys prior (cf. \cite{agresti_bayesian_2005}). The uniform prior corresponds to setting the components of the hyperparameters $\bm{\gamma}$ and $\bm{\alpha}_g$, $g \in H$, equal to 1 in Theorem~\ref{thm_like_post}. The Jeffreys prior, defined as proportional to the square root of the determinant of the Fisher Information Matrix (FIM), corresponds (in the case of a single categorical stochastic variable) to setting all components of $\bm{\alpha}$ equal to $1/2$ (\cite[303]{agresti_bayesian_2005}). However, due to the repeated occurrence of $\beta_g$ in the likelihood function \eqref{eq_likelihood}, we find a slightly different outcome which we did not find in recent text books.

\begin{proposition}\label{prop_jeffreys}
    The Jeffreys prior for the likelihood function \eqref{eq_likelihood} is given by a product of $K+1$ independent Dirichlet distributions of order $K$ with hyperparameters $\alpha_{gh} = 1/2$ for all $g, h \in H$ and $\gamma_g = K/2$ for all $g \in H$.
\end{proposition}

\begin{proof}
    The proof consists of computing the determinant of the Fisher information matrix (FIM). To compute the FIM, a linearly independent set of parameters specifying the model is required. Note that the parameter vectors $\{\bm{p}_g: g \in H\}$ and $\bm{\beta}$ are elements of the $K$-simplex $\Delta_K$, as $p_{gg} = 1 - \sum_{h \neq g} p_{gh}$ for each $g \in H$ and $\beta_K = 1- \sum_{g \neq K} \beta_g$, where we identify $H$ with the set $\{1, 2, \ldots, K\}$. It follows that the classification error model has $(K+1)(K-1) = K^2-1$ free parameters, which we will denote by $\widetilde{\bm{p}} = (\widetilde{\bm{p}}_g)_{g\in H}$, with $\widetilde{\bm{p}}_g = (p_{gh})_{h\neq g}$, and $\widetilde{\bm{\beta}} = (\beta_g)_{g \neq K}$. It is straightforward to show that the FIM is a block-diagonal matrix of the form
    \begin{equation}\label{eq_methods_FIM}
        \bm{I}(\widetilde{\bm{p}}, \widetilde{\bm{\beta}}) =
        \begin{pmatrix}
            A_1 & \bm{0} & \cdots & \bm{0} \\
            \bm{0} & \ddots & & \vdots\\
            \vdots & & A_K & \bm{0}\\
            \bm{0} & \cdots & \bm{0} & A_{K+1}
        \end{pmatrix},
    \end{equation}
    where $A_g = \beta_g \left[p_{gg}^{-1} \cdot\bm{1}\cdot\bm{1}^T + \diag(\widetilde{\bm{p}}_g)^{-1}\right]$ for $g \in H$ and $A_{K+1} = \beta_K^{-1}\cdot\bm{1}\cdot\bm{1}^T + \diag(\widetilde{\bm{\beta}})^{-1}$, see \cite[p.~391]{hogg_introduction_2018}. We will now use the fact that the determinant of an $m \times m$ matrix $C = \alpha\bm{a}\bm{b}^T + \diag(\bm{d})$, with $\alpha \in \mathbb{R}$ and $\bm{a},\bm{b},\bm{d} \in \mathbb{R}^m$ is given by
    \begin{equation}
        \det(C) = \left(1 + \alpha \sum_{j=1}^m \frac{a_jb_j}{d_j}\right)\prod_{j=1}^m d_j.
    \end{equation}
    For a proof, consult \cite[pp.~293-294]{graybill_matrices_1983}. Applied to \eqref{eq_methods_FIM}, we find $\det(A_g) = \beta_g^{K-1}\prod_{h\in H} p_{gh}^{-1}$, for $g \in H$ and $\det(A_{K+1}) = \prod_{g \in H} \beta_g^{-1}$. It follows that
    \begin{equation}\label{eq_methods_FIM_det}
        \det(\bm{I}(\bm{p}, \bm{\beta})) = \left(\prod_{g, h \in H} p_{gh}^{-1}\right)\left(\prod_{g\in H} \beta_g^{K-2}\right).
    \end{equation}
    Taking the square root concludes the proof.
\end{proof}

The Jeffreys prior density for the binary classification problem ($K=2$) reduces to
\begin{equation}
    p(\bm{p}, \bm{\beta} \mid \bm{\alpha}, \bm{\gamma}) = \frac{1}{\pi^2\sqrt{p(1-p)q(1-q)}},
\end{equation}
where $p = p_{10}$ and $q = p_{01}$.
Panel (a) of Figure~\ref{fig_methods_post_convergence} shows the marginal Jeffreys prior density for the parameter $p = p_{10}$ (1 minus recall) for $K=2$. Panels (b)---(d) show the marginal posterior density for $n=50$, $n=500$ and $n=2,\!000$, using simulated data with base rate $\beta_1 = 0.1$ and true classification error rates $p = 0.3$ and $q = 0.1$. The posterior density indeed converges to the true parameter value $p = 0.3$ (dashed line).

\begin{figure}[tb]
    \centering
    \subfloat[$n=0$]{\includegraphics[width=0.25\textwidth]{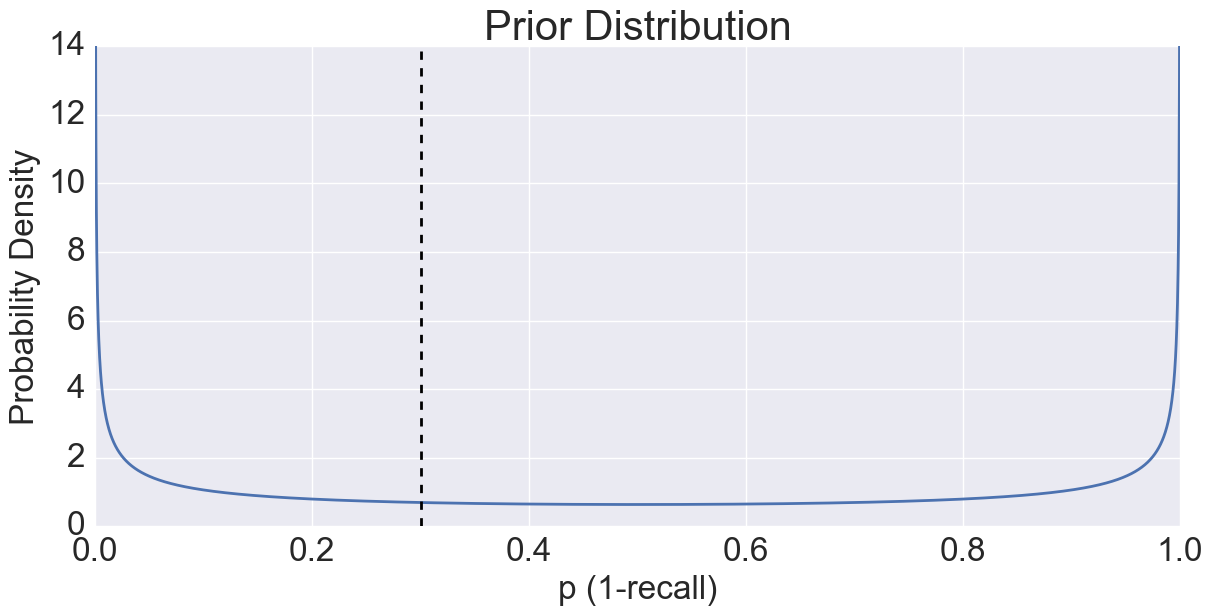}}
    \subfloat[$n=50$]{\includegraphics[width=0.25\textwidth]{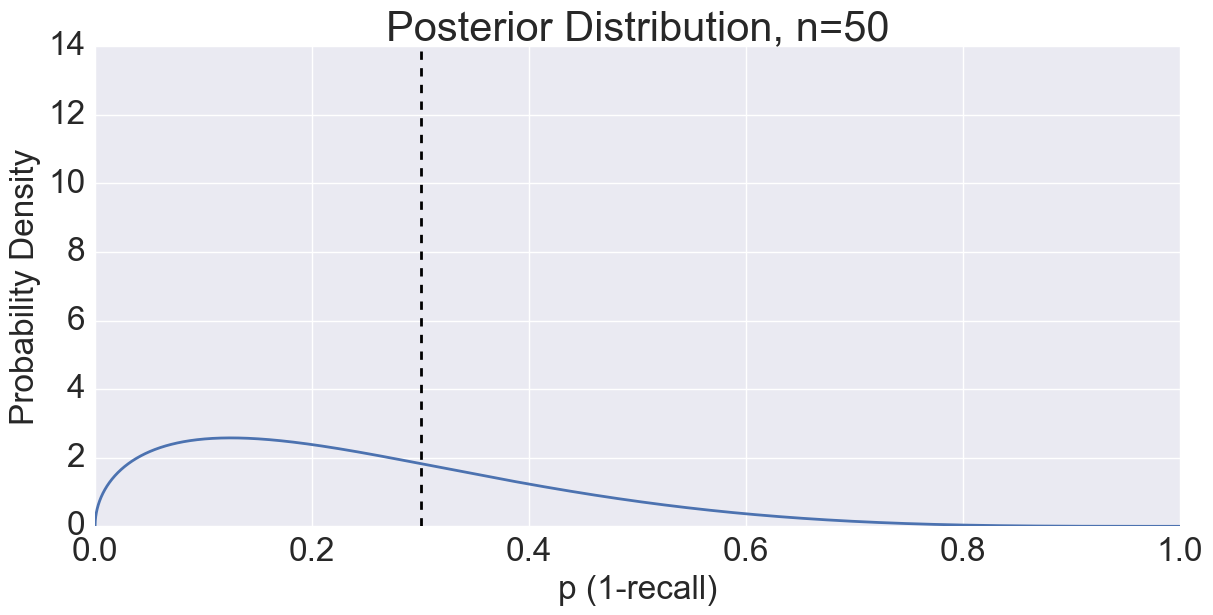}}\\
    \subfloat[$n=500$]{\includegraphics[width=0.25\textwidth]{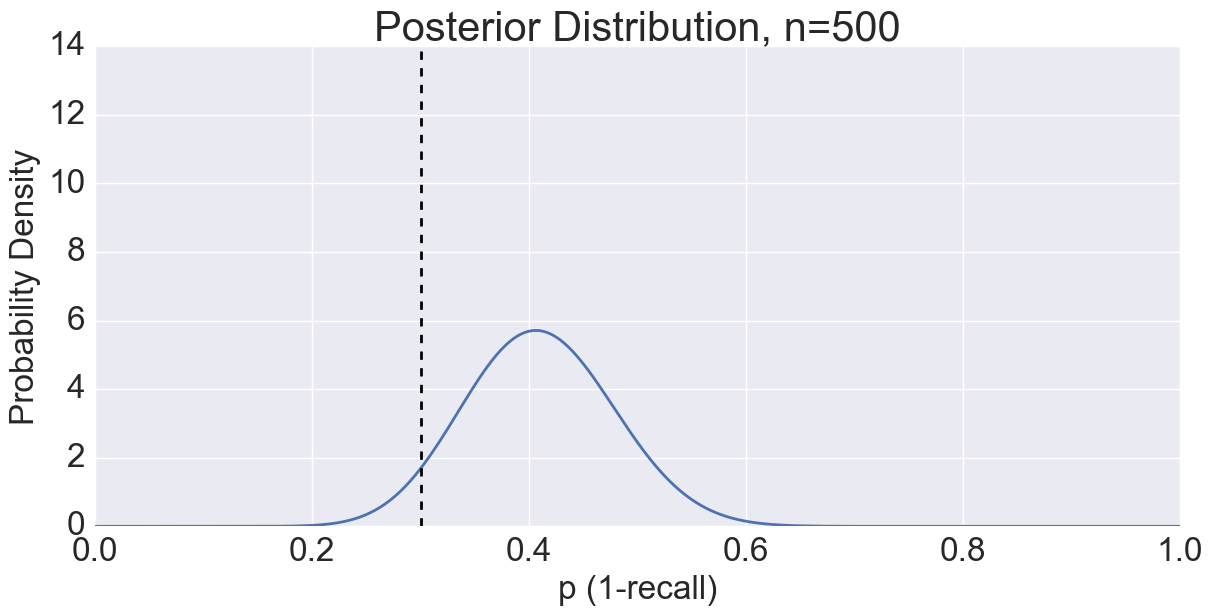}}
    \subfloat[$n=2,\!000$]{\includegraphics[width=0.25\textwidth]{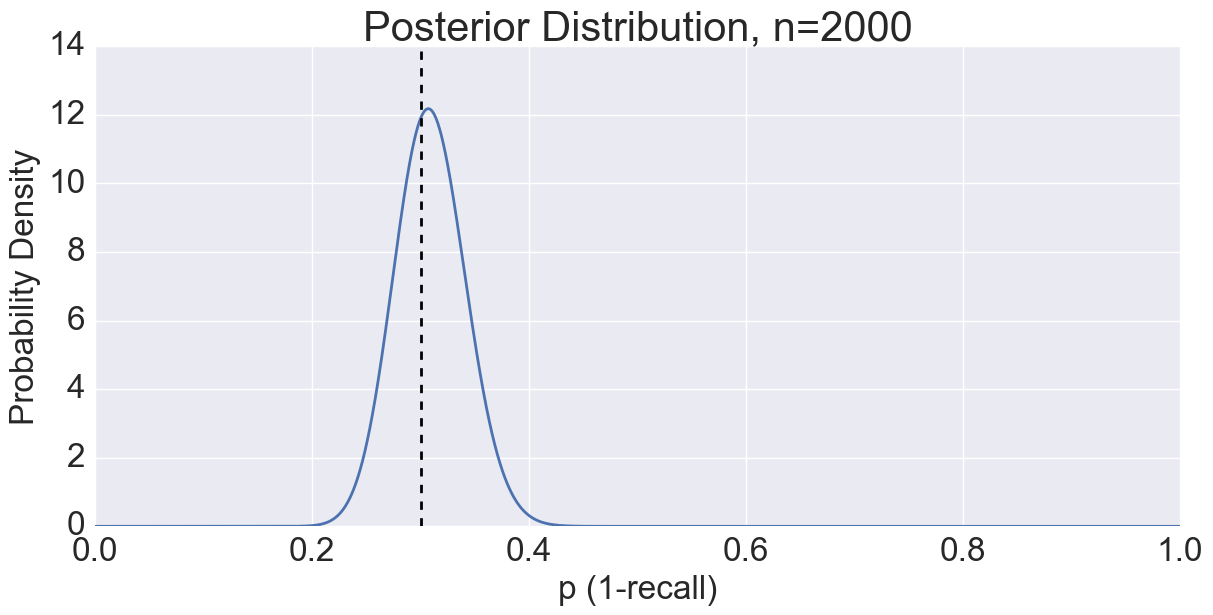}}
    \caption{\textit{The Jeffreys prior (no parameter constraints) and resulting posterior densities for model $p$, converging to the true parameter $p_0 = 0.3$ as $n$ increases.}
    }
    \label{fig_methods_post_convergence}
\end{figure}

\subsection{Imposing Parameter Constraints.}\label{sec_methods_bayes_restr}
Motivated by The Peculiar Example, we deduce useful constraints for the model parameter $\bm{p}$. The example showed that in (very) small test sets, the unbiased estimator $Q\widehat{\bm{v}}$ of the counts vector $\bm{v}$ might give impermissible estimates (negative counts). In particular, we estimated the true webshop count to be $-75$.

In our setting, there are two possible explanations for finding a negative estimate of a positive quantity: (I) the predicted base rate vector, being only a single realization of a stochastic variable, lies far away from its mean in this particular case, or (II) we have estimated the contingency matrix not sufficiently accurately and therefore the bias correction is inaccurate. The following theorem demonstrates why, in many practical cases, (II) plays a larger role than (I).


\begin{theorem}\label{thm_base_rate_variance}
    The variance-covariance matrix of the base rate vector $\widehat{\bm{\beta}}$, conditional on the true $s_i$, equals
    \begin{equation}
        \Var(\widehat{\bm{\beta}} \mid \{s_i, i\in I\}) = \frac{1}{N}
        \left(
            \diag(P^T\bm{\beta}) -  P^T\diag(\bm{\beta})P 
        \right).
    \end{equation}
\end{theorem}

\begin{proof}
    Recall the $N\times K$ class matrix $A = (a_{ih})$ given by $a_{ih} = I(s_i = h)$ for objects $i\in I$ and classes $h \in H$. The vector $\bm{a_i}$ denoted the $i$-th column of $A^T$. Note that the equality $\bm{a_i}\bm{a_i}^T = \diag(\bm{a_i})$ holds, as each $\bm{a_i}$ is a standard basis vector of $\R^K$. A similar equality holds for $\bm{\widehat{a}_i}$. The variance-covariance matrix of $\bm{\widehat{a}_i}$ (for readability, we leave out the conditionality on the right-hand side) is equal to
    \begin{align*}
        \Var(\bm{\widehat{a}_i} \mid \{s_i, i\in I\})
        &= \E(\bm{\widehat{a}_i}\bm{\widehat{a}_i}^T) - \E(\bm{\widehat{a}_i} )\E(\bm{\widehat{a}_i})^T\\
        &= \E(\diag(\bm{\widehat{a}_i})) - P^T \bm{a_i}\bm{a_i}^T P\\
        &= \diag(P^T\bm{a_i}) - P^T \diag(\bm{a_i}) P.
    \end{align*}
    In the last equality, we used the fact that the operation $\diag(\cdot)$ commutes with taking the expectation. Recall that $\bm{\widehat{a}_i}$ and $\bm{\widehat{a}_j}$ ($i\neq j$) are independent, conditional on $\{s_i, i\in I\}$. Thus, summing both sides of the above equation over $i\in I$ and dividing the results by $N^2$ concludes the proof.
\end{proof}

Theorem~\ref{thm_base_rate_variance} shows that the variance of $\widehat{\bm{\beta}}$ is proportional to the inverse of the \textit{population size} $N$, while the variance of the parameter $\bm{p}$ is proportional to the inverse \textit{test set size} $n$. In practice, we often find $N \gg n$, hence (II) plays a larger role than (I).

Therefore, we wish to impose constraints on the model parameter $\bm{p}$ such that $Q\widehat{\bm{v}} \geq \bm{0}$ with probability~1, conditional on $\widehat{\bm{v}}$. Given $\widehat{\bm{v}}$, we write
\begin{equation}\label{eq_restr_general}
    \Theta_K = \Theta_K(\widehat{\bm{v}}) \coloneqq \left\{\bm{p} \in (\Delta_{K-1})^K : Q(\bm{p}) \widehat{\bm{v}} \geq \bm{0}\right\},
\end{equation}
where $\Delta_{K-1}$ is the standard $(K-1)$-simplex in $\mathbb{R}^K$. In other words, to determine whether a point $\bm{p}$ is in $\Theta_K$, one has to check whether or not $\widehat{\bm{v}}$ is contained in the convex hull of the rows of the contingency matrix $P(\bm{p})$ corresponding to the point $\bm{p}$.

\begin{figure}[tb]
    \centering
    \subfloat[]{\includegraphics[width=0.18\textwidth]{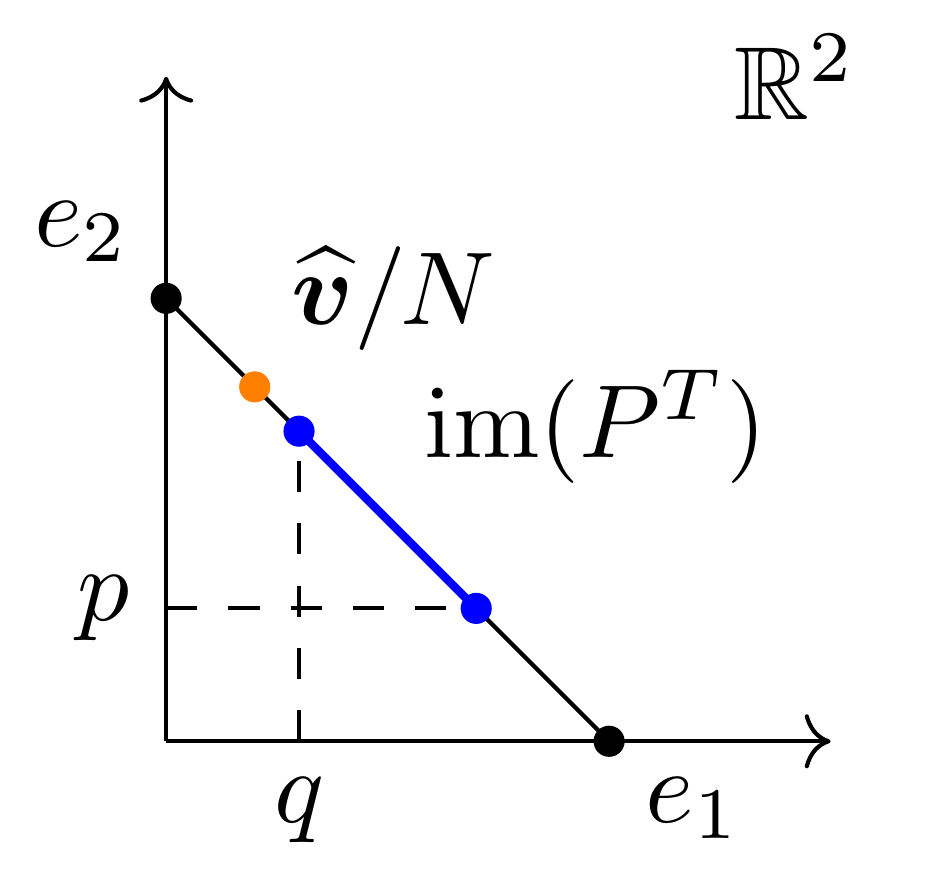}}\hspace{2cm}%
    \subfloat[]{\includegraphics[width=0.18\textwidth]{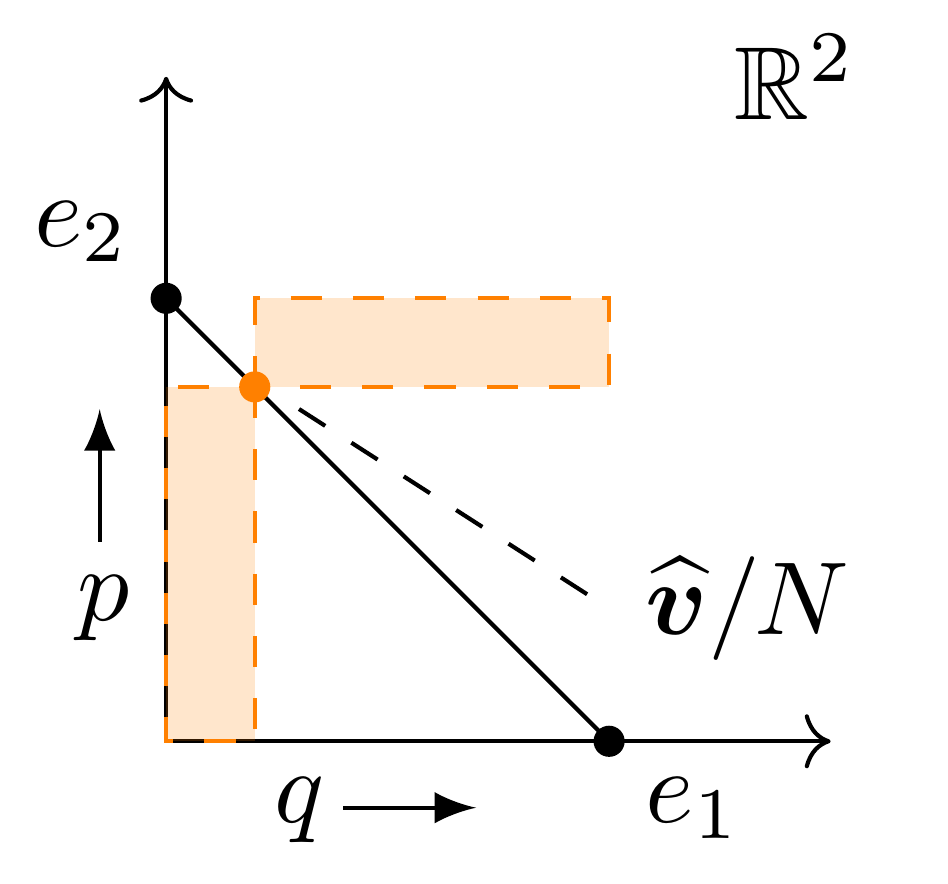}}
    \caption{\textit{The geometric relations between the contingency matrix $P$, the  estimated counts vector $\widehat{\bm{v}}$ and the size $N$ of the unseen data are shown in panel (a), for $K=2$.
    Panel (b) shows the constraints on the model parameters $p$ and $q$ that we impose.}}
    \label{fig_methods_bayes_restr}
\end{figure}

\vspace{\baselineskip} \noindent \textbf{The Binary Case.} In the binary case of $K=2$, the proposed parameter constraints take on the elegant form
\begin{equation}\label{eq_restr_binary}
    \Theta_2(\widehat{\bm{v}}) \cong ([0,\widehat v_2] \times [0, \widehat v_1]) \cup ([\widehat v_2, 1] \times [\widehat v_1,1]).
\end{equation}
To prove equation~\eqref{eq_restr_binary} algebraically, one solves the two linear equations $Q(\bm{p}) \widehat{\bm{v}} \geq \bm{0}$ for $p=p_{10}$ and $q=p_{01}$. Instead of doing so, we prefer a more insightful geometrical proof, see Figure~\ref{fig_methods_bayes_restr}, in which $e_1$ and $e_2$ correspond to the standard basis vectors of $\mathbb{R}^2$. Note that the predicted base rate vector $\widehat{\bm{\beta}} = \widehat{\bm{v}}/N$ (orange dot) lies on the 1-simplex $\Delta_1 \subset \mathbb{R}^2$. The blue line segment in Figure~\ref{fig_methods_bayes_restr}, panel (a), shows the image of $\Delta_1$ under $P^T$ and the corresponding locations of $p$ and $q$ on the axes. It shows precisely what we found in The Peculiar Example: the predicted counts vector does not lie in the image of $\Delta_1$ under the \textit{estimated} transposed contingency matrix $P^T$.
Now, Theorem~\ref{thm_base_rate_variance} shows that it is unlikely that the predicted base rate vector $\widehat{\bm{\beta}}$ is not contained in the image of the \textit{true} transposed contingency matrix.
In Figure~\ref{fig_methods_bayes_restr}, panel (a), we thus impose that we need a blue endpoint of $\text{im}(P^T)$ on each side of the orange dot $\widehat{\bm{v}}/N$. It follows that the orange area in Figure~\ref{fig_methods_bayes_restr}, panel (b), contains the permitted $(q,p)$-pairs and thus corresponds to $\Theta_2(\widehat{\bm{v}})$. This concludes the geometrical proof of equation~\eqref{eq_restr_binary}.

\subsection{Bayesian Aggregates.}\label{sec_methods_agg}
In \cite{van_delden_accuracy_2016}, it was shown that $\E[\widehat{\bm{u}}] = P^T\bm{u}$. Hence, an unbiased estimator for $\bm{u}$ would be $Q\widehat{\bm{u}}$ with $Q = (P^T)^{-1}$. Recall that we refer to this method as the baseline method. The baseline method does not yet take the uncertainty in estimating the contingency matrix $P$ into account. In The Peculiar Example, we have seen that this might lead to impermissible outcomes for small test sets. We propose the following three-step approach, assuming that a classification algorithm has already been trained and used to predict the classifications on the entire unseen dataset. In step~1, choose a prior from the Dirichlet family, such as the Jeffreys prior in Proposition~\ref{prop_jeffreys}, and find the posterior distribution of $\bm{p}$ using Theorem~\ref{thm_like_post} (integrating $\bm{\beta}$ out of the equation). Alternatively, MCMC methods can be used to obtain a numerical approximation of the posterior distribution, also for non-conjugate (but proper) priors. In step~2, take draws from the posterior distribution and only accept the draws that lie within $\Theta_K(\widehat{\bm{v}})$. Choose a positive integer $R$ (resolution parameter) and stop step 2 after $R$ draws have been obtained. In step~3, compute the matrix $Q$ and the product $Q\widehat{\bm{u}}$, for each of the $R$ draws. These three steps give a numerical approximation (of resolution $R$) to the posterior distribution of the estimator of the aggregate vector $\bm{u}$, conditional on $\widehat{\bm{u}}$.

We conclude by noting that the time complexity of the proposed Bayesian method in $n$ and $N$ is equal to that of the baseline method if conjugate priors are considered, as it allows direct sampling from the posterior distribution. In fact, if wall-clock time is considered, the entire estimation can be easily implemented to finish in tens of milliseconds on any regular machine. Using MCMC methods, the time complexity depends on the prior choice and might increase with $n$, but not with $N$. The wall-clock time for our bias correction method using MCMC methods will increase to several minutes.


\section{Empirical Evaluation}\label{sec_exper}

We begin by briefly revisiting The Peculiar Example from Section~\ref{sec_intro}. Then, we introduce real-world data on company tax returns and use it to compare existing methods to our bias correction method.

\subsection{The Peculiar Example.}\label{sec_res_pec_ex}
Recall The Peculiar Example from Section~\ref{sec_intro}. The baseline method resulted in an estimate of $-75$ webshops. The result of our bias correction method, with Jeffreys prior and parameter constraints, is shown in Figure~\ref{fig_res_comp_freq_pec}. The posterior mean of the number of items in the webshop class is now equal to $5.0$. The distribution is skewed (to the right) and the estimator still has a large variance, indicating that more labeled data is required to obtain a more accurate estimate. However, our bias correction method \textit{is} able to capture the little information available in the small test set of size $n=10$, resulting in a useful estimate.

\begin{figure}[tb]
    \centering
    \includegraphics[width = 0.5\textwidth]{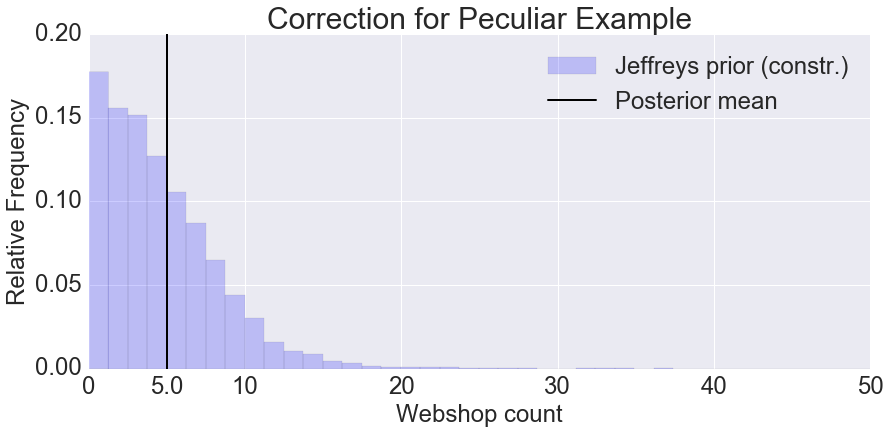}
    \caption{\textit{The results of our method (with Jeffreys prior and parameter constraints) for The Peculiar Example.}}
    \label{fig_res_comp_freq_pec}
\end{figure}

\begin{figure}[tb]
    \centering
    \includegraphics[width=0.5\textwidth]{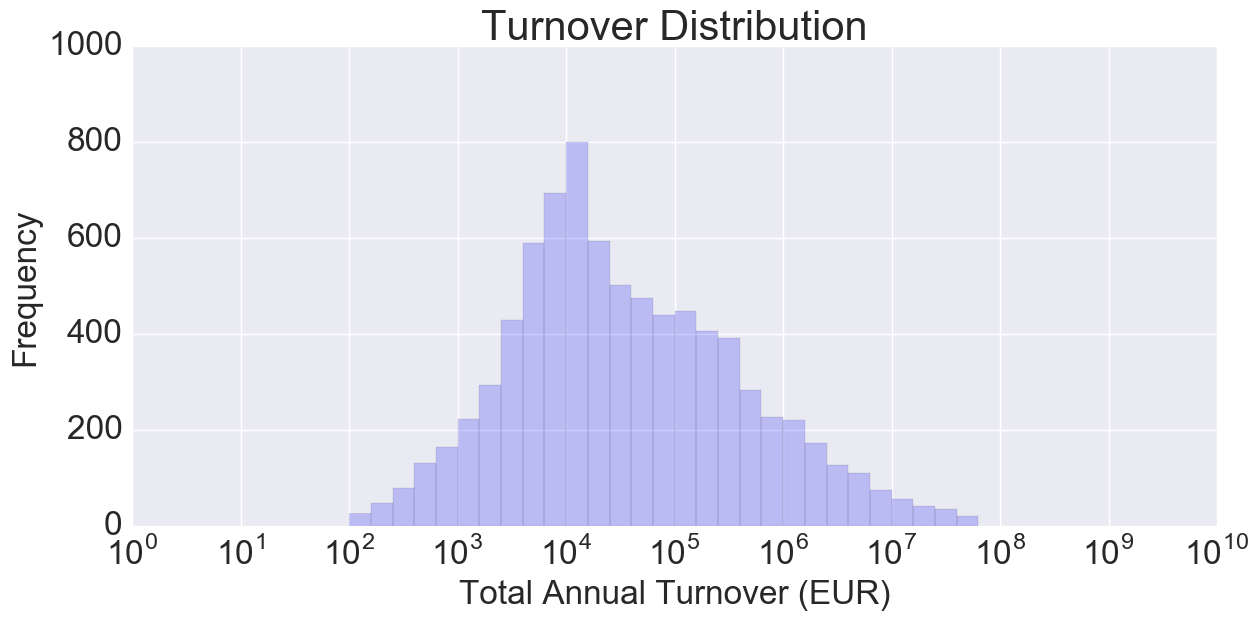}
    \caption{\textit{The distribution of the turnover variable in a dataset of filed tax returns. Figure adapted from \cite{meertens_data-driven_2018}.}}
    \label{fig_res_data}
\end{figure}

\subsection{Data on Company Turnover.}\label{sec_res_data}
To show the strength of our bias correction method in a general setting, we will now empirically evaluate our bias correction method using a real-world dataset of company tax returns. The goal is to estimate the total annual turnover of webshops. The binary classification (webshop or not) is not available in the data, but the annual turnover $y_i$ for each company $i$ is. This application is exhaustive, because the numerical variable $y$ (company turnover) is not constant (i.e., $y_i$ depends on $i$).

\begin{figure}[tb]
    \centering
    \subfloat[$n = 50$]{\includegraphics[width=0.5\textwidth]{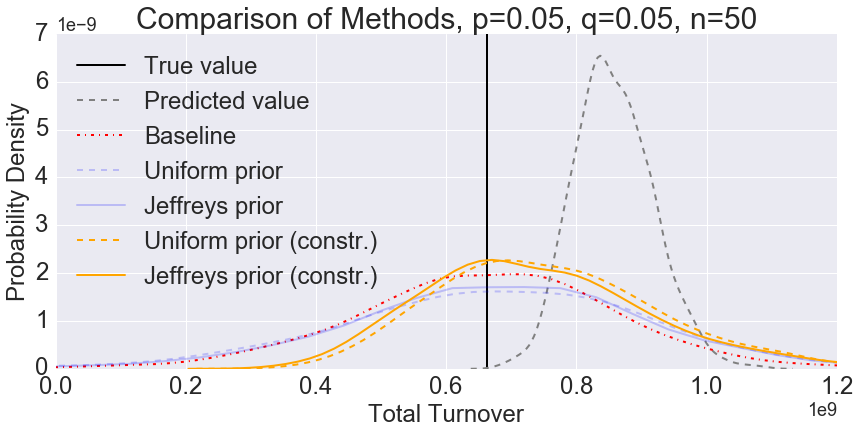}}\\
    \subfloat[$n=2,000$]{\includegraphics[width=0.5\textwidth]{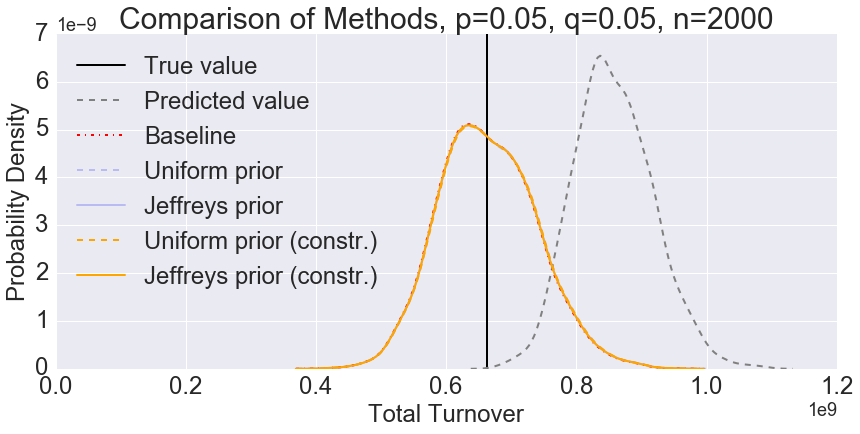}}
    \caption{\textit{The distribution of the mean of the posterior distribution of total turnover for different priors and different test set sizes.}}
    \label{fig_res_bias_comparison}
\end{figure}


The dataset, earlier studied in \cite{meertens_data-driven_2018}, contains tax returns filed in the Netherlands in 2016 by a set $I$ of $N~=~18,\!939$ companies established outside the Netherlands, but within the European Union. The total turnover reported by these companies for 2016 equals EUR~12.2~billion. Figure~\ref{fig_res_data} shows how companies' total annual turnover is distributed in the dataset.

In \cite{meertens_data-driven_2018}, a classification algorithm was trained (on a separate dataset) to predict whether a company was active as a webshop (class 1) or not (class 0). Without discussing the details of training it, we run the classification algorithm on the full dataset of size $N$, obtaining the predicted classifications $\widehat{s}_i \in \{1, 0\}$.

\subsection{Bias Correction and Effect of Prior Beliefs.}\label{sec_res_corr}
We compare the accuracy of our bias correction method to that of the baseline method and that of Bayesian methods without parameter constraints, when applied to estimate the total turnover of webshops in the given dataset. To compute the accuracy of the estimators, we have to know the true classification $s_i$ (webshop or not) for each company. As the true classifications are unknown, the \textit{only} way to examine the accuracy is by means of simulation. For the purpose of this simulation, we take the predicted classifications from \cite{meertens_data-driven_2018} as the true classifications $s_i$. The base rate of the webshop class then equals $\beta_1 = 0.075$.

Next, we take the binary classification error model with the true classification error rates $p = 0.05$ and $q = 0.05$ as the data generating process. We make this choice of $p$ and $q$, because it leads to a sufficiently large bias of classification-based aggregates (as $q/(p+q) = 0.5 \gg 0.075 = \beta_1$) making Figure \ref{fig_res_bias_comparison} more pleasant to read. We emphasize that other choices of $p$ and $q$ will lead to similar conclusions. 

\begin{table}
    \scriptsize{}
    \begin{tabular}{lrrrrrr}
        \toprule
         & \multicolumn{3}{c}{$n=50$} & \multicolumn{3}{c}{$n=2,\!000$}\\
        \cmidrule(r){2-4}\cmidrule(r){5-7}
         & Bias & Var & MSE & Bias & Var & MSE \\
         \textbf{Method} & $\times 10^6$ & $\times 10^{15}$ & $\times 10^{15}$ & $\times 10^6$ & $\times 10^{15}$ & $\times 10^{15}$ \\
        \midrule
        None & 195,3 & 3,7 & 41,8 & 195,3 & 3,7 & 41,8 \\
        Baseline & -1,3 & 46,3 & 46,3 & 1,3 & 5,4 & 5,4 \\
        \midrule
        Unif. & 10,8 & 72,7 & 72,8 & 3,1 & 5,5 & 5,5 \\
        Jeffr. & 22,0 & 81,7 & 82,2 & 2,3 & 5,5 & 5,5 \\
        \midrule
        Unif. con. & 105,1 & 35,8 & 46,9 & 3,1 & 5,5 & 5,5 \\
        Jeffr. con. & 88,5 & 39,0 & 46,9 & 2,3 & 5,5 & 5,5 \\
        \bottomrule
    \end{tabular}
    \caption{\textit{A comparison of the bias, variance and mean squared error (MSE) of each of the bias correction methods (including no correction) when estimating webshop turnover.}}
    \label{tab_res_bias_comparison}
\end{table}

The results of our Monte Carlo simulation are shown in Figure~\ref{fig_res_bias_comparison}, including the true value of webshop turnover $u_1$ (black, solid) and the distribution of the predicted value $\widehat{u}_1$ without performing any corrections (black, dashed). Besides our bias correction method (orange, dashed and solid), the figure shows the baseline method (red, dashed) and existing Bayesian bias correction methods without parameter constraints (blue, dashed and solid). The distributions are obtained by drawing $1,\!000$ bootstrap replications from the classification error model and following the approach proposed in Section~\ref{sec_methods_agg} for each bootstrap replication. Panel (a) shows the results for a test set of size $n=50$ and panel (b) shows the results for $n = 2,\!000$.

To facilitate a more rigorous comparison of the methods, the same results are summarized in Table~\ref{tab_res_bias_comparison}. We make four observations from the results. First, for $n=50$, our bias correction method achieves a considerable reduction of mean squared error (MSE) compared to existing Bayesian methods without parameter constraints. The effect of imposing parameter constraints diminishes for $n=2,\!000$. Second, for $n=2,\!000$, our bias correction method performs equally well (in terms of MSE) as the baseline method and existing Bayesian methods. Third, also for $n=2,\!000$, all bias correction methods are a huge improvement (in terms of MSE) compared to performing no bias correction; the bias substantially decreases without increasing the variance too much.  Fourth, our bias correction method decreases the variance compared to the baseline method, essentially by cutting off a large part of the support (cf. Figure \ref{fig_res_bias_comparison}, panel (a)). Even though this slightly increases MSE, it guarantees that impermissible estimates are never found, as we have illustrated in Section~\ref{sec_res_pec_ex}.

Finally, we note that a dataset of size $N~=~18,\!939$ is rather small for data mining applications nowadays. However, the base rate example and equation~\eqref{eq:main_problem} showed that the relative bias does not depend on $N$. Therefore, the experimental results that we have found for the dataset of company tax returns will generalize to (much) larger datasets in other applications.


\section{Conclusion}\label{sec_concl}

In this paper, we have studied the statistical bias of classification-based aggregates in the general setting of multi-class classification. We proposed a Bayesian bias correction method, being the first to derive and impose constraints on the classification error rates.

For small test sets, imposing these parameter constraints dismisses impermissible estimates, leading to similar MSE as the baseline method and reduced MSE compared to existing Bayesian methods. For larger test sets, all bias correction methods yield similar results, substantially reducing the MSE of classification-based aggregates. The improvement of our method compared to existing methods is particularly compelling if the number of labeled data points is small. We argue this to be relevant in many data mining applications, including sentiment analysis and land cover mapping.

As future work, we aim to reduce MSE even further by studying the bias-variance trade-off in classification-based aggregates in more detail, and by relaxing the assumption of homogeneity of contingency matrices.


\section*{Acknowledgements}

This research has been funded by Statistics Netherlands and the Dutch Ministry of Economic Affairs. We express our gratitude to Arnout van Delden, Sander Scholtus and Danny van Elswijk for insightful discussions and for useful comments on an earlier version of the text.

\bibliographystyle{siam}

\bibliography{bayesian}

\end{document}